\documentclass{article}

\newcommand{\Bfalse}{\bot}

\newcommand{\Band}{\,\wedge\,}
\newcommand{\Bor}{\,\vee\,}
\newcommand{\Bnot}{\neg}
\newcommand{\Bimp}{\,\rightarrow\,}

\newcommand{\Bool}{\mathbb{B}}

\newcommand{\st}{\,|\,}

\usepackage{tikz}
\usepackage{microtype}
\usepackage{graphicx}
\usepackage{subfigure}
\usepackage{booktabs} 
\usepackage{subcaption}
\usepackage{comment}
\usepackage{hyperref}
\usepackage{pgfplots}
\usepackage{xcolor}
\usepackage{tcolorbox}

\usepackage[accepted]{icml2025}

\usepackage{amsmath}
\usepackage{amssymb}
\usepackage{mathtools}
\usepackage{amsthm}

\usepackage[capitalize,noabbrev]{cleveref}

\usepackage{pgfplots}
\pgfplotsset{compat=1.18}
\usepackage{caption}
\usepackage{pgfplots}
\usepgfplotslibrary{fillbetween}
\usepackage{outlines}

\theoremstyle{plain}
\newtheorem{theorem}{Theorem}[section]

\theoremstyle{definition}

\theoremstyle{remark}

\usepackage[textsize=tiny]{todonotes}

\icmltitlerunning{Can LLMs Learn Formal Logic?}

\begin{document}

\twocolumn[
\icmltitle{Can Large Language Models Learn Formal Logic? \\ A Data-Driven Training and Evaluation Framework}




\begin{icmlauthorlist}
\icmlauthor{Yuan Xia}{usc}
\icmlauthor{Akanksha Atrey}{nbl}
\icmlauthor{Fadoua Khmaissia}{nbl}
\icmlauthor{Kedar S. Namjoshi}{nbl}
\end{icmlauthorlist}

\icmlaffiliation{usc}{University of Southern California, Los Angeles, CA, USA}
\icmlaffiliation{nbl}{Nokia Bell Labs, Murray Hill, NJ, USA}

\icmlcorrespondingauthor{Yuan Xia}{yuanxia@usc.edu}

\icmlkeywords{logical reasoning, large language models, Hilbert proofs, reinforcement learning with validator feedback}

\vskip 0.3in
]



\printAffiliationsAndNotice{}

\begin{abstract}
This paper investigates the logical reasoning capabilities of large language models (LLMs). For a precisely defined yet tractable formulation, we choose the conceptually simple but technically complex task of constructing proofs in Boolean logic. A trained LLM receives as input a set of assumptions and a goal, and produces as output a proof that formally derives the goal from the assumptions. Incorrect proofs are caught by an automated proof checker. A critical obstacle for training is the scarcity of real-world proofs. We propose an efficient, randomized procedure for synthesizing valid proofs and introduce Template Transformation, a data augmentation technique that enhances the model's ability to handle complex logical expressions. 
The central evaluation question is whether an LLM has indeed learned to reason. We propose  tests to measure the reasoning ability of a black-box LLM. By these measures, experiments demonstrate strong reasoning capabilities for assertions with short proofs, which decline with proof complexity. Notably, template transformation improves accuracy even for smaller models, suggesting its effectiveness across model scales.


\end{abstract}

\section{Introduction}

In July 2024, The New York Times published an article aptly titled “\textit{A.I. Can Write Poetry, but It Struggles With Math.}”~\cite{nyt-ai-poetry-2024}. General Large Language Models (LLMs) are indeed excellent at manipulating text but can be poor at logical reasoning, which is at the heart of mathematics and computing. This work aims to investigate experimentally whether that is an inherent limitation.

LLM-based reasoning falls into two broad categories. In neuro-symbolic systems such as AlphaProof~\cite{alpha-proof-2024}, the LLM guides a conventional proof-search engine, replacing hard-coded heuristics or hints provided by humans. In the second category, which our work belongs to, are LLMs that search for proofs directly. Chain-of-Thought (CoT)~\cite{DBLP:conf/nips/Wei0SBIXCLZ22} is one such system. 

The thought-chain produced by a CoT method is expressed in natural language. While easier for humans to understand, a natural-language proof cannot be automatically validated precisely, due to ambiguities inherent in natural language. 



In order to systematically investigate the reasoning capabilities of LLMs, it is simpler and more effective to express input assertions and output proofs in formal language. A ``reasoning'' LLM (Fig.~\ref{fig:l3m_pipeline}) receives as input a set of assumptions and a goal assertion, both expressed in a formal logic. As output, the LLM exhibits a formal proof that explains how the goal follows logically from the assumptions. That is, an LLM acts as a proof generator. 

There are good reasons to believe that LLMs can act as proof generators and good reasons to believe that they cannot. On the positive side, LLMs are Turing-complete, assuming infinite precision arithmetic~\cite{DBLP:journals/corr/abs-1901-03429,DBLP:conf/conll/BhattamishraPG20,DBLP:journals/jmlr/PerezBM21,DBLP:journals/tacl/StroblMW0A24}. An LLM scans its input and past output to determine its next output symbol, using attention mechanisms to focus on the relevant portions. This is similar to how, in constructing a mathematical proof, humans select and combine past deductions to form new conclusions. Few-shot learning suggests that LLMs can recognize patterns and follow general rules. On the negative side, Turing completeness relies on unbounded integers (to encode an unbounded Turing machine tape); but in reality, machine integers have bounded precision. LLMs have bounded context, no auxiliary memory, and no backtracking ability; hence, colloquially, LLMs may ``forget'' the distant past. This is an issue for proof construction, as an assumption may not be required until a late stage of proof. LLMs operate statistically and may thus hallucinate and produce incorrect proofs even if the claim is valid. Finally, logical reasoning has a inherently high worst-case complexity: it is co-NP-complete even for Boolean logic, with rapidly increasing difficulty (e.g., RE-complete for first-order logic).

To precisely formulate and test reasoning ability, we focus on reasoning in Boolean logic. In our view, there are several factors that make this a good choice. Boolean logic is the basis for other logics (e.g., first order, modal, temporal, and the like). Therefore, if Boolean reasoning turns out to be a challenging task for LLMs, it is unlikely that LLMs can do better on more sophisticated reasoning. Boolean logic has a simple proof system that makes it easy to automatically validate proofs. Semantically, Boolean reasoning is non-trivial, being co-NP-complete. (The dual question of Boolean satisfiability is the canonical NP-complete problem.)

In a strong sense, therefore, Boolean reasoning can be viewed as the \emph{C.elegans}~\cite{brenner-nobel-2002} of LLM reasoning: structurally simple, semantically non-trivial, and exhibiting "in the small" the central difficulties of general reasoning. 

The central challenges are (1) a scarcity of real-world training data, i.e., valid proofs, and (2) a means of determining whether a black-box model is reasoning rather than memorizing or guessing. We address these critical questions (1) by giving an efficient randomized algorithm for generating a large and diverse corpus of logically valid proofs for training and testing, and (2) by designing meaningful, automated, black-box tests of reasoning ability. 
The precision of formal language and the ability to automatically validate candidate proofs simplifies both tasks, enabling them to be carried out fully automatically, efficiently, and at scale. 



Our experiments show that Llama models can be trained to reason well in Boolean logic through a combination of synthetic proof generation and proof augmentation. Using only 9000 synthetically generated proofs as training data, our fine-tuned Llama-8B model reaches 98\% accuracy on depth-7 proofs, outperforming GPT-4o's few-shot learning performance (70\% accuracy), despite having far fewer parameters. Our ablation studies reveal two key findings: larger model capacity enables better generalization to deeper proofs, while proof augmentation  enhances the models' ability to handle complex instances, particularly benefiting the smaller Llama-1B model. 

In our view, this work makes several contributions. One is the recognition that training an LLM over a formal logic has distinct advantages: synthetic data can be generated at large scale and variety for training, and reasoning ability can be evaluated robustly with automated tests that are generated and run at scale. The second is the choice of Boolean logic to study reasoning ability. If LLMs fail at Boolean logic, it is unlikely that they will succeed at more complex reasoning. The third is a technical contribution of template proof transformations, which produces syntactically varied proofs that share a common structure. In training, this forces the LLM to focus on abstract proof patterns rather than on concrete tokens, improving its generalization ability. A fourth contribution is the variety of proposed tests that can be applied to determine whether a black-box LLM is actually reasoning vs. guessing or retrieving answers.

\section{Background}\label{sec:background}




In this section, we define the syntax and semantics of Boolean logic, introduce the Hilbert proof system, and forward and backward proof-construction mechanisms.

\subsection{Boolean Logic: Syntax and Semantics}

Boolean formulas (also called assertions or expressions) are formed by combining Boolean variables that represent atomic propositions using familiar combinators, such as ``and'' ($\Band$), ``or'' ($\Bor)$, ``not'' ($\Bnot$), and ``implies'' ($\Bimp$). The Boolean domain is denoted $\Bool$. 

Given a countable set $X$ of Boolean variables (also called propositions), a Boolean formula over $X$ is formed by the minimal grammar $f := p \in X \st \Bnot f \st f \Band f$. 
Disjunction is defined by $f \Bor g = \Bnot( (\Bnot f) \Band (\Bnot g))$, and implication by $f \Bimp g = (\Bnot f) \Bor g$.

An \emph{assignment} is a function from $X$ to $\Bool$; i.e., it provides a Boolean value for each variable. The \emph{satisfaction} of a formula by an assignment $\pi$ is a relation $\models$ defined recursively by: $\pi \models p$ iff $\pi(p)$ is true; $\pi \models \Bnot f$ iff it is not the case that $\pi \models f$; and $\pi \models f \Band g$ iff $\pi \models f$ and $\pi \models g$.

For a Boolean formula, the \emph{satisfiability} question is whether there is a assignment satisfying the formula. This is the classic NP-complete question. The dual \emph{validity} question is whether \emph{all} assignments satisfy the formula. This is the canonical co-NP-complete question.

\subsection{Hilbert Proofs}

An instance of a Boolean reasoning problem is given by a (possibly empty) set of assumption formulas and a goal formula. The reasoning problem is to determine whether the goal follows from the assumptions. A proof is a step-by-step explanation that justifies a "yes" answer.  

There are several proof systems for Boolean logic. In this work, we use the well known Hilbert style of proof system~\cite{mendelson-book-2024}. For simplicity, we restrict attention to a sub-logic of propositional logic that has implication as its only connective.\footnote{In this sub-logic, every formula is satisfiable, but validity (i.e., reasoning) is still co-NP-hard (Appendix \ref{appendix:proofs}).} The Hilbert-style proof system for this logic consists of two axioms and one inference rule, called Modus Ponens (MP).

The axiom schemas are the following. 
\begin{align*}
  & \mathit{A1::} \quad (A \Bimp (B \Bimp A)) \\
  & \mathit{A2::} \quad (A \Bimp (B \Bimp C)) \Bimp ((A \Bimp B) \Bimp (A \Bimp C)) 
\end{align*}
Colloquially, the first axiom says that if both $A$ and $B$ hold, then $A$ holds. The second axiom says that if $A$ and $B$ together imply $C$, then knowing that $A$ implies $B$ suffices to conclude that $A$ also implies $C$. 

An \emph{instance} of an axiom schema is obtained by providing a substitution that specifies formulas for the placeholders $A,B,C$. 
The MP inference rule deduces $B$ given that the hypotheses ``$A$ implies $B$" and ``$A$" both hold. 

Formally, a \emph{proof} is a sequence of steps where each step consists of a claimed formula along with a \emph{justification}. A justification for step $k$ of the proof can be of three kinds: (1) matching the $k$'th formula to an assumption; (2) matching the $k$'th formula to an axiom instance with a substitution $\sigma$, which is correct if the $k$'th formula is the result of applying $\sigma$ to the named axiom; 
or (3) MP applied to steps $m,n$ of the proof, which is correct if, denoting the $k$'th formula by $g$, for some formula $f$, the $m$'th formula is $(f \Bimp g)$, and the $n$'th formula is $f$. Both $m,n$ must be strictly smaller than $k$ to avoid circular reasoning. 

A proof is \emph{valid} if every justification in the proof is correct and the goal formula is claimed at some proof step. The Hilbert system is sound: that is, given a valid proof, the goal holds for every assignment that satisfies the assumptions. 



\subsubsection{Proof Example [Transitivity]} Given assumptions $p \Bimp q$ and $q \Bimp r$, does the goal $p \Bimp r$ hold? This is just transitivity, which is true, as can be established by the following Hilbert proof:
\begin{enumerate}
\item $(q \Bimp r) \Bimp (p \Bimp (q \Bimp r))$, by Axiom A1 with substitution $A=q \Bimp r, B=p$.  
\item $q \Bimp r$, by Assumption $2$.   
\item $p \Bimp (q \Bimp r)$, by MP on steps $1,2$.
\item $(p \Bimp (q \Bimp r)) \Bimp ((p \Bimp q) \Bimp (p \Bimp r))$, by Axiom A2 with substitution $A=p, B=q, C=r$. 
\item $(p \Bimp q) \Bimp (p \Bimp r)$, by MP on steps $4,3$.
\item $p \Bimp q$, by Assumption $1$.
\item $p \Bimp r$, by MP on steps $5,6$. 
\end{enumerate}

Proofs in the Hilbert-style proof system have a well-deserved reputation for being difficult to follow for humans. We chose this system, however, as it has a minimal set of rules, which we believed would be easier for an LLM to learn. Other standard proof systems, such as natural deduction, produce proofs that are easier to follow~\cite{sep-natural-deduction} but have multiple inference rules. 

\subsubsection{Proof Construction}
There are two main mechanisms of proof construction. 

\emph{Forward proof construction} maintains a set of deduced formulas; initially, this set contains the supplied hypotheses. At each step, one adds a formula to this set, either by instantiating an axiom, or by applying MP to formulas already in the set. Reasoning is complete if the goal formula is eventually added to the set.

 \emph{Backwards} or \emph{Goal-directed} proof construction maintains a proof tree. Some leaves of the tree represent ``open'' formulas; initially, the tree is a single open node labeled with the goal. At a  backward step, one ``closes'' an open leaf node $n$ labeled with a formula $g$ in one of three ways: by recognizing $g$ as a hypothesis, or by recognizing $g$ as an axiom instance, or by introducing a formula $f$ and asserting that $g$ follows from $f$ and $(f \Bimp g)$ by MP -- these formulas label new open leaf nodes that are added as children of $n$. Reasoning is complete when all leaves are closed.

\section{Training LLMs to Reason}

A central challenge in training an LLM to reason with formal proofs is the shortage of data. Proofs are found in textbooks, or in proof repositories for automated theorem provers. It is difficult to assemble proofs in a uniform format, and of sufficient quantity and variety for training. 

We describe a simple and efficient randomized backward proof generation process, which synthesizes valid proofs of arbitrary size and depth. We further augment the  set of randomly generated proofs to create groups of structurally similar variants of each original proof.

Augmentation is useful in training, to ensure that the LLM focuses less on the syntax of a proof and more on the underlying logical reasoning -- in essence, we would like the LLM to learn  \emph{proof patterns} rather than memorize specific proofs. Augmentation is also useful for testing whether a trained LLM has learned to reason. 


\subsection{Proof Generation}

The proof generation process is, in essence, a randomized form of the goal-directed proof construction process. Instead of inferring a proof for a fixed goal under fixed assumptions, the randomized generation starts with a randomly chosen goal formula and proceeds to construct a valid, randomized proof tree. 


\begin{enumerate}
\item The root of the tree is labeled with a goal formula, say $g$, that is chosen at random.
  
\item At an leaf node labeled with a  formula $h$, the process decides at random whether to do one of the following. 
  \begin{itemize}
  \item \label{gen:stop} Stop expansion at this node. The open formula $h$ turns into an assumption used to prove $g$.
  \item Choose a formula $f$ at random and expand using a MP justification into two new subgoal nodes labeled with formulas $f$ and $f \Bimp h$.
  \item If $h$ is an axiom instance, close this node with an axiom justification.
  \end{itemize}
  
\item The process continues until the proof tree reaches a specified limit on size or depth.
\end{enumerate}

A proof tree generated in this way is turned into a valid proof by viewing all formulas of stopped open nodes (case~\ref{gen:stop}) as assumptions required to prove the goal $g$. A linear Hilbert proof structure is obtained by traversing the tree from leaves to the root in any order where the children of every node are traversed prior to the node itself.

Goal-directed proof generation is efficient for a Hilbert-style system as the generation process determines an appropriate resolution for every open formula locally, e.g., by randomizing the application of MP, in time that \emph{does not depend} on the  size of the current proof tree. In contrast, a forward generation process must search at each step for already-deduced formulas that match the shapes $f\Bimp g$ and $f$ to apply the MP rule. A na\"ive algorithm requires quadratic time in the number of deduced formulas, which can be reduced to near-linear time using hashing. Thus, goal-directed proof generation has a significant efficiency advantage, which extends to proof-generation in other systems such as natural deduction. 




\subsection{Proof Augmentation by Template Transformation}\label{sec:templateT}

The randomized proof-generation algorithm produces valid proofs, but does not  guarantee that many examples with same or similar proof structure are generated. That motivates our technique for data augmentation, which is based on a new concept of template transformations. 

This technique is based on the key observation that a proof in Hilbert form may also be viewed as a proof schema. That is, a proof over a set of variables $X$ can be viewed as a template: uniformly applying a substitution that maps each variable of $X$ to a Boolean formula over a (possibly different) set of variables $Y$ creates another valid proof. 

For instance, the reasoning example of Section~\ref{sec:background} is defined over variables $p,q,r$. Say we uniformly replace   $q$ with $y_1$, and $p$ with $y_1 \Bimp y_2$, and $r$ with $y_2 \Bimp y_1$. This replacement creates a valid proof showing that the new goal $(y_1 \Bimp y_2) \Bimp (y_2 \Bimp y_1)$ (transformed from the original goal $p \Bimp r$) follows by transitivity from the transformed assumptions $(y_1 \Bimp y_2) \Bimp y_1 $ and $y_1 \Bimp (y_2 \Bimp y_1)$. 

In general, to transform a valid reasoning instance defined over a set $X$ of variables, we (1) choose a substitution $\sigma$ \emph{at random} that maps each variable in $X$ to a randomly chosen Boolean expression over a randomly chosen variable set $Y$ (not necessarily different from $X$), (2) systematically replace each formula $f$ that occurs in the proof with the formula $\sigma(f)$ that is obtained by applying the substitution $\sigma$ to $f$, and (3) systematically apply the substitution $\sigma$ to the assumption and goal formulas. 

It is easy to show by induction on proof length that the transformed reasoning instance is also valid. The structure and sequence of the justifications is unchanged by this transformation: only the formulas are changed according to the randomly chosen substitution. Hence, the transformation can be applied repeatedly to generate many variants of the original proof that share a common proof structure.  This construction is used to train the LLM model and (as discussed in Section~\ref{sec:testing}) also used to test the reasoning ability of a trained LLM model.

The intuition is that, during training, supplying similar instances obtained through template transformation forces the LLM to focus on the abstract proof structure  rather than on concrete detail such as variable names. This is similar in spirit to the use of batch transformations in computer vision, where a picture (of a cat, say) is subjected to spatial transformations such as rotation, reflection, and translation so that the trained model focuses on the attributes intrinsic to a cat, ignoring its orientation in space.

\section{Evaluation Framework}\label{sec:testing}
In this section, we present evaluation metrics to assess the reasoning capabilities of LLMs and demonstrate the effectiveness of our training and evaluation methods.


Existing LLM evaluation metrics focus on different aspects of model output, such as quality, relevancy, factual correctness, bias, and toxicity, but there is no standardized approach for assessing  reasoning capabilities. Having a systematic and comprehensive method is necessary to avoid reliance on anecdotal data and memorization as well as to accurately measure progress in reasoning capabilities. 

We propose an evaluation process with four clearly defined and measurable metrics. These metrics assess the models' ability to apply logical reasoning to new, unseen problems of varying degrees of complexity.

\subsection{Favoring Semantics over Syntax} 
An LLM should demonstrate the ability to generalize beyond the specific instances in its training data. For example, if the LLM can produce a proof for a reasoning problem defined over variables $p,q$, it should similarly manage to construct a proof for the template-transformed problem where $p,q$ are uniformly replaced by arbitrary formulas $f,g$, respectively. This indicates that the LLM recognizes and abstracts the underlying logic of the problem, disregarding superficial syntactic differences such as variable names or structural variances. 

\subsection{Generalizing to Deeper Proofs}
A reasoning LLM should discover proofs that are longer than those present in its training set. This suggests that the LLM is not merely memorizing proofs up to a certain complexity but exhibiting deeper logical reasoning. The evaluation of this metric involves measuring model accuracy with increasing depth of reasoning, highlighting its ability to navigate and resolve more complex proof structures.

\subsection{Handling Wider Expressions}
An LLM should prove statements that incorporate a larger set of variables than those found in its training dataset. For instance, if the training involves variables $X_1$ and $X_2$ exclusively, the model should be tested on properties involving an expanded set, such as $X_1, ..., X_n$. Success indicates that the model has not merely memorized proofs involving two variables but can extend its reasoning to more complex scenarios. Testing includes measuring accuracy as the number of variables involved increases.

\subsection{Exhibiting Diversity in Proofs}
A logical statement can have several valid proofs. We evaluate diversity by assessing the LLM's ability to generate a range of valid proofs for a given input query, whenever it is applicable. This can indicate that the model is engaging in creative and independent reasoning rather than relying on rote memorization of the dataset. To support this metric, we use a symbolic proof checker to \emph{semantically} evaluate proof validity rather than relying on \emph{syntactic} token-level equality. 

\section{Experimental Analysis}
\definecolor{BlueGrad1}{HTML}{7791BD}
\definecolor{mycolor1}{HTML}{8DD3C7}
\definecolor{mycolor2}{HTML}{5470A3}
\definecolor{mycolor3}{HTML}{BEBADA}
\definecolor{mycolor4}{HTML}{FB8072}
\definecolor{mycolor5}{HTML}{80B1D3}
\definecolor{mycolor6}{HTML}{FDB462}
\definecolor{mycolor7}{HTML}{B3DE69}
\definecolor{mycolor8}{HTML}{FCCDE5}
\definecolor{mycolor9}{HTML}{6580B0}
\definecolor{mycolor10}{HTML}{BC80BD}
\definecolor{mycolor11}{HTML}{CCEBC5}
\definecolor{mycolor12}{HTML}{FFED6F}

\begin{table*}[h]
\caption{Proof Generation Accuracy on In-Distribution Test Data: Comparison between fine-tuned Llama models (trained on 9000 examples), GPT-4o few-shot learning (averaged across 1,3,6,9-shot learning settings), and pre-trained baseline models. Testing performed on proof depths ($d$) matching training distribution. w/ TT and w/o TT denote models fine-tuned with and without Template Transformation respectively.}
\label{tab:comparison}
\centering
\begin{tabular*}{0.9\textwidth}{@{\extracolsep{\fill}}lccc@{}}
\toprule
& \multicolumn{3}{c}{\textbf{Accuracy by Proof Depth}} \\
\textbf{Model} & \textbf{$d=7$} & \textbf{$d=10$} & \textbf{$d=13$} \\
\midrule
\multicolumn{4}{l}{\textit{SOTA Performance}} \\
GPT-4o (X-shot) & 0.67±0.08 & 0.48±0.08 & 0.26±0.04 \\
\midrule
\multicolumn{4}{l}{\textit{Pre-trained Base Models}} \\
Llama-3.2-8B (X-shot) & $\approx$ 0.00 & $\approx$ 0.00 & $\approx$ 0.00\\
Llama-3.1-1B (X-shot) & $\approx$ 0.00 & $\approx$ 0.00 & $\approx$ 0.00\\

\midrule
\multicolumn{4}{l}{\textit{Our Approach - Small Model (Llama-3.1-1B)}} \\
Fine-tuned w/o TT & 0.52 & 0.30 & 0.02 \\
Fine-tuned w/ TT & 0.54 & 0.30 & 0.04 \\
\midrule
\multicolumn{4}{l}{\textit{Our Approach - Large Model (Llama-3.2-8B)}} \\
Fine-tuned w/o TT &  0.78 & 0.40 & 0.28  \\
Fine-tuned w/ TT & \textbf{0.98} & \textbf{0.66} & \textbf{0.32} \\
\bottomrule
\\
\end{tabular*}
\end{table*}
\subsection{Experimental Setup}

\begin{figure}[t]
    \centering
\includegraphics[width=\columnwidth]{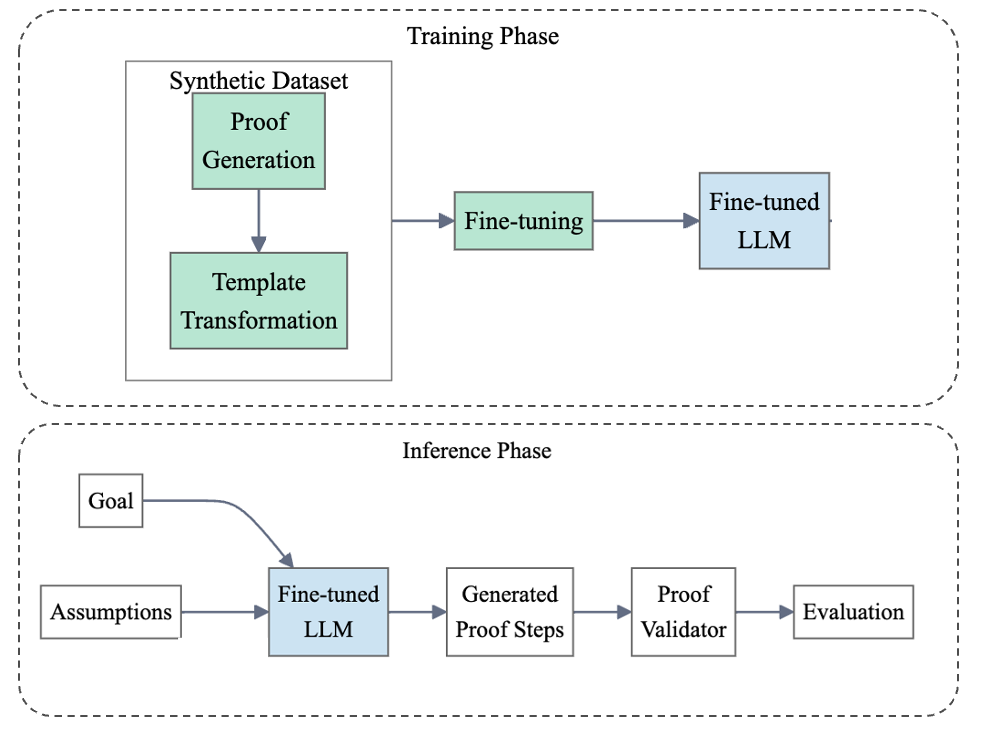}
    \caption{Overview: The training phase generates synthetic proofs and applies template transformations for fine-tuning. Inference produces a candidate proof for the query (``Does the goal follow from the assumptions?"); this proof is formally validated.}
    \label{fig:l3m_pipeline}
\end{figure}
\paragraph{Training} We fine-tune two Llama3 models (8B and 1B parameters) \cite{dubey2024llama} on our synthetically generated Hilbert proofs. We use 4-bit quantization to optimize memory usage. The fine-tuning process employs Low-Rank Adaptation (LoRA) \cite{hu2021lora} to efficiently update model parameters without modifying the entire network. Training convergence is monitored using exact match accuracy with the ground-truth proofs. We also use early stopping (patience = 5 epochs) to prevent overfitting.

For all experiments, the training data consists of a balanced number of proofs with depths 7, 10, and 13. During training, Template Transformation is applied with a probability $\alpha_{TT} =0.7$ during batch preprocessing. Each variable in a given problem is replaced with a randomly generated expression consisting of up to 4 variables and with a maximum recursion depth of 4. Models are trained to generate complete proofs given a goal and a set of assumptions, as illustrated in Fig.~\ref{fig:l3m_pipeline}. 

\textbf{Evaluation} We generate separate proof trees for training and evaluation to prevent data leakage.  To obtain the evaluation dataset, we construct a few distinct proof trees and ensure all formulas $f$ are unique without duplicates occurring in the training dataset (Section ~\ref{sec:templateT}). From these trees, we extract both the validation and testing sets. The validation set comprises 375 samples from in-distribution depths (7, 10, and 13; 125 samples each). 

For generalization analysis, we use 450 test samples across nine depths (4-28; 50 samples per depth) and 400 samples across four widths of increasing variable complexity (100 samples per width). We ensure unique formulas across all partitions.
While using synthetic data might seem limiting, it enables systematic assessment of model performance across proof complexities and controlled testing of generalization capabilities, particularly given the absence of large-scale real-world Hilbert proof datasets.

To evaluate our models, we report the testing accuracy defined as the ratio of correctly generated proofs to the total number of test problems. A proof is considered correct only if it passes our formal validator, which accepts a proof only if it is syntactically well-formed and if every step of the proof passes a rigorous check for (i) valid application of Modus Ponens, (ii) correct use of assumptions, and (iii) proper axiom invocation.

\textbf{Experimental Design} Our experimental analysis is twofold: First, we establish baseline performance by comparing our fine-tuned models against GPT-4o \cite{gpt4} and pre-trained Llama3 models on in-distribution proof depths. Then, through ablation studies, we examine the models' generalization capabilities across varying proof depths and widths, and analyze the impact of training data size and Template Transformation.

\textbf{Key Findings} Our experiments reveal the effectiveness of fine-tuning Llama3 models (8B and 1B) for Hilbert proof generation using our proposed synthetic data and proof augmentation. The 8B model achieves 98\% accuracy on depth-7 proofs, outperforming GPT-4o's few-shot performance (67\%±8 averaged across 1,3,6,9-shot settings) despite having far fewer parameters. Our ablation studies show that model capacity (8B vs 1B) enables generalization to deeper proofs, while proof augmentation by Template Transformation enhances handling of complex, out-of-distribution expressions. Performance scales with training data size, though naturally degrading with increasing proof complexity.

\subsection{Baseline Evaluation}

We evaluate our approach against GPT-4o (100B+ parameters) and compare against our base models: Llama-3.2 (8B) and Llama-3.1 (1B). As shown in Table~\ref{tab:comparison}, testing is performed on proof depths matching the training distribution \{7,10,13\}. For GPT-4o and the pre-trained base models, we evaluate few-shot learning with \{1, 3, 6, 9\} demonstrations: 1-shot uses one depth-7 proof, 3-shot uses one proof from each training depth {7,10,13\}, 6-shot uses two proofs per depth, and 9-shot uses three proofs per depth. We provide some examples in Appendix ~\ref{sec:appendix2}.
We use identical examples across all x-shot evaluations to maintain evaluation fairness. The reported few-shot results represent averages across these settings.

GPT-4o achieves an average accuracy of 0.67±0.08 on depth-7 proofs, with performance degrading to 0.48±0.08 and 0.26±0.04 at depths 10 and 13. 
The pre-trained Llama3 base models fail to generate valid proofs ($\approx 0.00$ accuracy) when tested directly on our task, likely due to both their smaller model sizes compared to GPT-4o ($>100$B parameters) and the challenge of inferring strict syntactic rules of Hilbert proofs from few examples. 

Our fine-tuned models show substantial improvements over both GPT-4o and their pre-trained versions. The large model (8B) with Template Transformation achieves 0.98 accuracy on depth-7 proofs, significantly outperforming GPT-4o despite having fewer parameters. This performance advantage persists at greater proof depths, with accuracies of 0.66 and 0.32 for depths 10 and 13 respectively. Template Transformation proves particularly effective for the large model, improving accuracy by 20-25\% mainly for shorter depths compared to standard fine-tuning.

The small model (1B) shows more modest improvements, achieving 0.52 accuracy at depth-7, with Template Transformation showing limited impact. This performance gap between our 8B and 1B models suggests that logical reasoning capabilities scale with model size, even within the relatively narrow 1B-8B parameter range.

\subsection{Ablation Study}
We conduct ablation studies to systematically analyze how different factors affect our fine-tuned LLMs' performance, and to understand the limits of our proposed approach.

\begin{figure}[h]
    \centering
    \includegraphics[width=\linewidth]{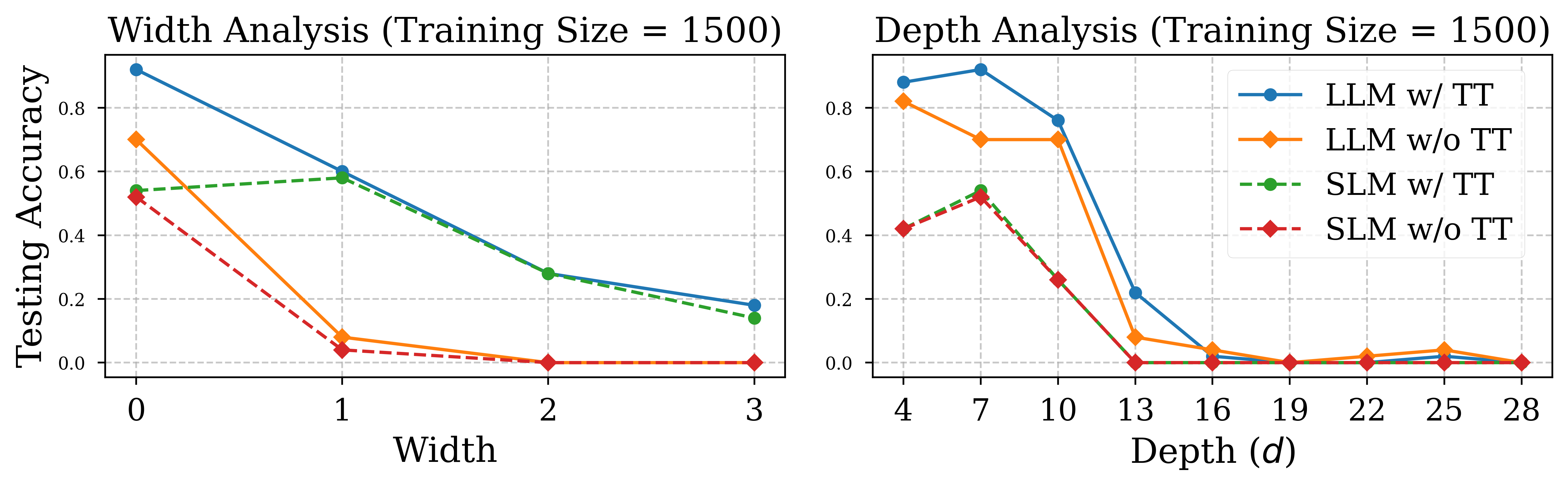}
    \includegraphics[width=\linewidth]{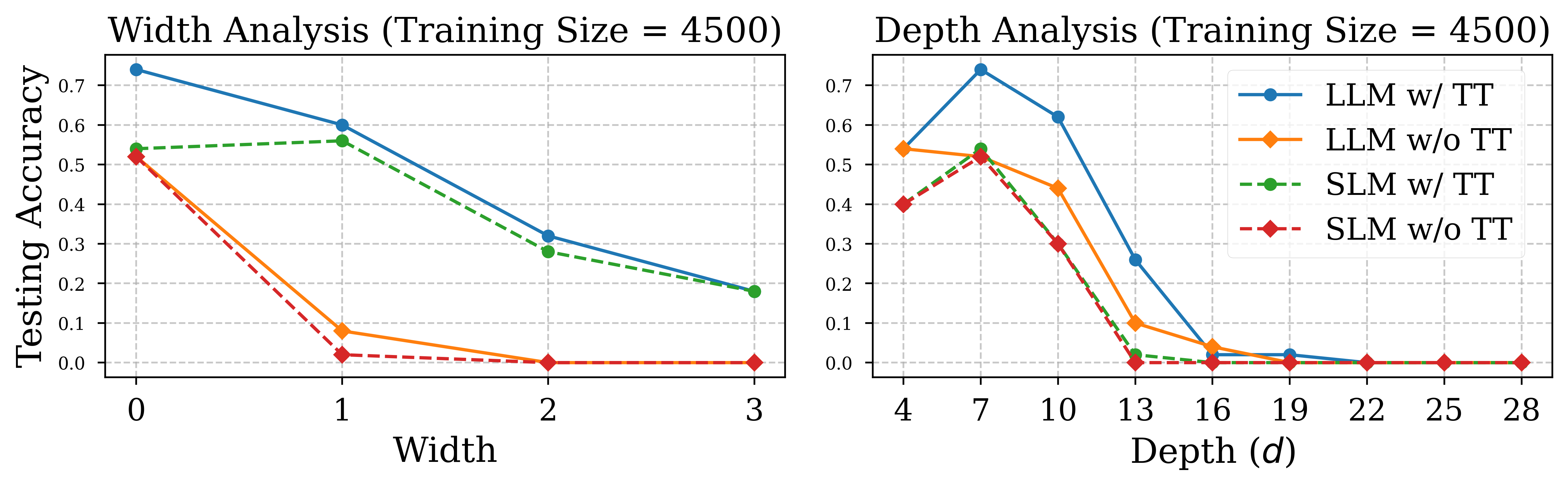}
    \includegraphics[width=\linewidth]{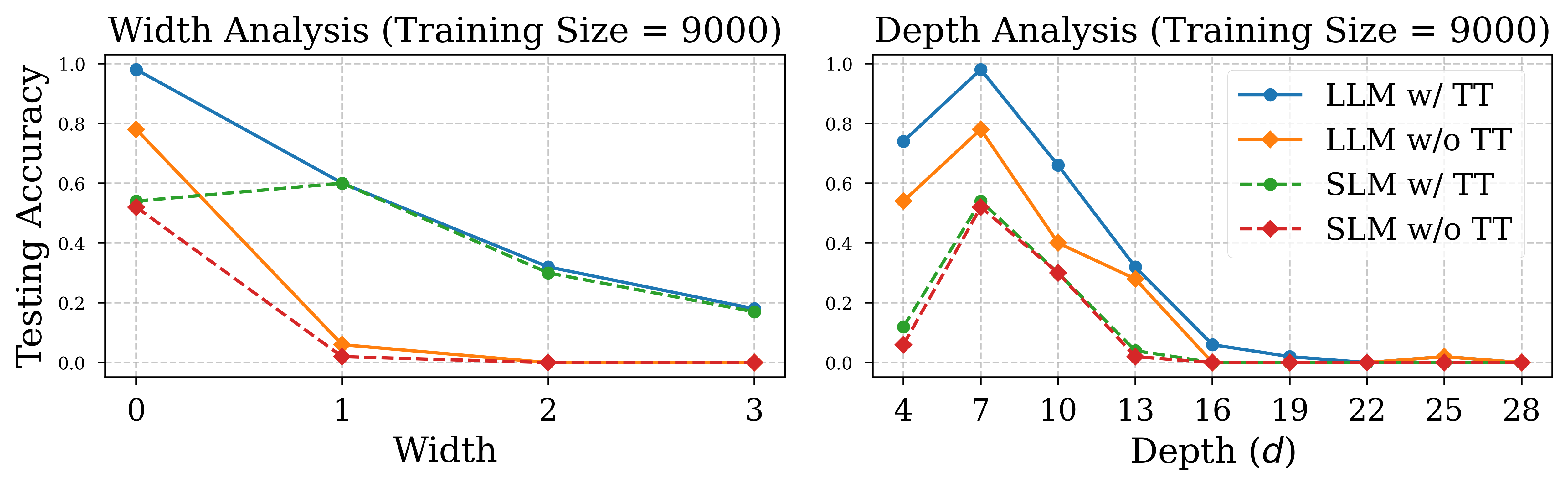}
   \caption{Comparison of Llama3-8B (LLM) and Llama3-1B (SLM) model variants' testing accuracies across different training data sizes. Left: Impact of proof width on accuracy. Right: Impact of proof depth on accuracy. Models with and without Template Transformation (TT) are compared. }
    \label{fig:analysis}
\end{figure}

\subsubsection{Impact of Proof Depth}
To evaluate reasoning depth generalization, we test models on nine proof depths (4-28 steps, 50 samples per depth) using fixed-width expressions. Each proof in the test set is generated by systematically reducing deeper proofs: identifying the first modus ponens application (i.e., one closest to the leaf nodes) and replacing it with the deduced statement, which is treated as a new assumption, to reduce depth by one step.

As shown in Fig.~\ref{fig:analysis}, performance degrades with increasing proof complexity, but larger models show better resilience. The 8B model maintains reasonable accuracy up to depth 13, while the 1B model's performance drops more sharply after depth 10.

\subsubsection{Impact of Proof Width}
We analyze width generalization using four levels of expression complexity (100 samples per width). Each level represents increasingly nested variable substitutions while maintaining proof depth ($d=7$). A width-0 proof uses simple variables (e.g., $P \Rightarrow Q$), while higher widths introduce nested implications (e.g., $(P \Rightarrow Q) \Rightarrow R$ for width-1).

Results in Fig.~\ref{fig:analysis} show that Template Transformation significantly improves handling of complex expressions for both the 8B and 1B models.  Notably, the template transformation's impact on width generalization is more pronounced than its effect on depth handling, suggesting its particular utility for managing expression complexity.


\subsubsection{Impact of Model Size and Training Data}
We compare Llama-8B and Llama-1B which we denote as LLM and SLM respectively, reflecting their relative model capacities. We train three instances of each model using synthetic  training sets of three sizes: 1500, 4500, and 9000 examples. Each set is balanced across proof depths 7, 10, and 13. 

Fig.~\ref{fig:analysis} summarizes the obtained results. While the LLM generally outperforms the SLM, particularly on deeper proofs, we observe an interesting pattern with proof width: when using Template Transformation, the SLM achieves comparable performance to the LLM on wider expressions ($width\geq 2$), suggesting that our data augmentation technique effectively compensates for smaller model capacity in handling complex expressions. The LLM, however, shows better data efficiency overall, achieving stronger performance with smaller training sets, especially on depth generalization.

\subsubsection{Impact of Template Transformation}
We analyze how the probability of applying Template Transformation ($\alpha_{TT}$) during the training process impacts the models' reasoning capability. We vary $\alpha_{TT}$ from 0 to 0.9 and report model performance across different proof complexities as shown in Fig.~\ref{fig:tt_accuracy}. 
\begin{figure}[h]
    \centering
\includegraphics[width=\columnwidth]{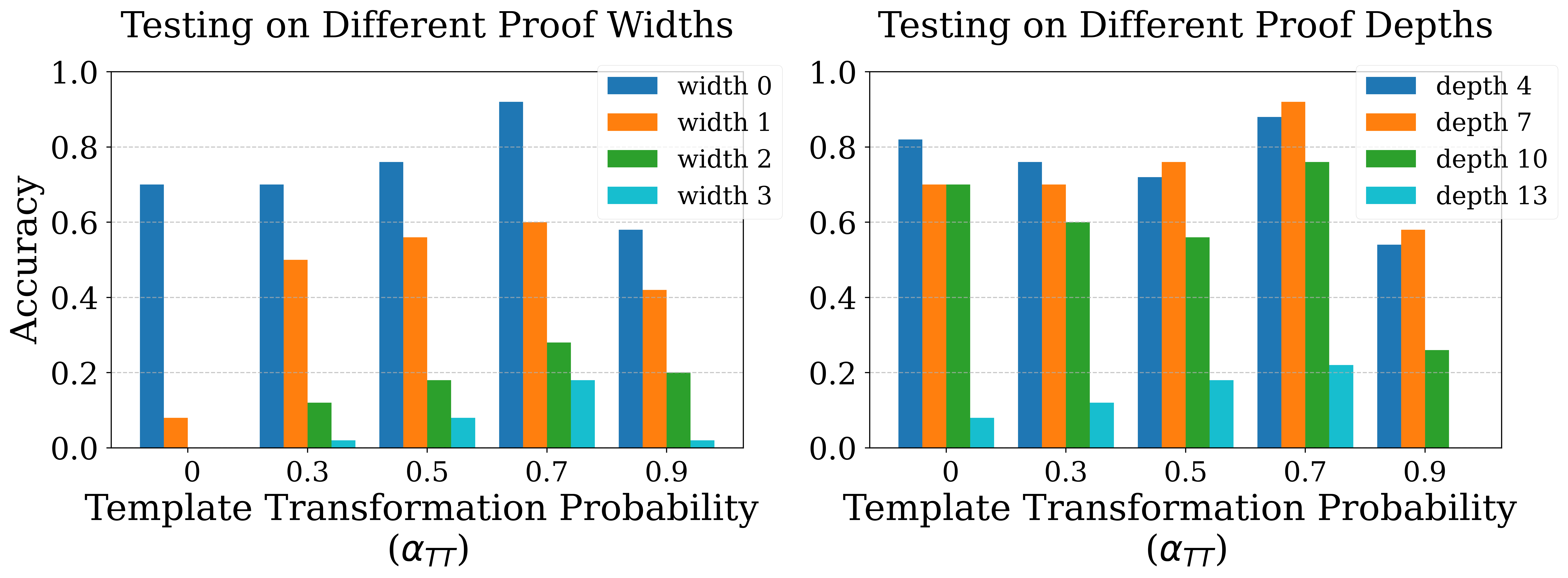}
    \caption{Impact of Template Transformation probability ($\alpha_{TT}$) on model performance. The left subplot shows test accuracy across different model widths (0-3), while the right subplot demonstrates the effect on models of varying depths (4-13).}
    \label{fig:tt_accuracy}
\end{figure}

The results reveal a consistent pattern for proof width: performance improves as $\alpha_{TT}$ increases up to 0.7, beyond which accuracy declines sharply. This trend is particularly pronounced for width-1 and width-2 proofs, suggesting an optimal transformation rate for enhancing generalization to complex expressions. For proof depth, the impact varies: while deeper proofs (depth 13) show similar sensitivity to $\alpha_{TT}$, shallower proofs exhibit more stable performance across different transformation rates. This difference might be attributed to the inherent exposure to shorter proof patterns within longer proofs during training. Setting $\alpha_{TT}=0.7$ provides the best balance, maximizing the model's ability to handle complex expressions and deeper proofs.


\section{Related Work}\label{sec:relwork}

LLM-based proof generation methods fall into two broad categories.
In neuro-symbolic systems such as AlphaProof~\cite{alpha-proof-2024},
AlphaGeometry~\cite{DBLP:journals/nature/TrinhWLHL24}, Leandojo~\cite{leandojo}, and GPT-f~\cite{DBLP:journals/corr/abs-2009-03393}, the LLM component generates hints that guide a conventionally programmed symbolic proof-search engine (hints such as “use this theorem," “split  into these cases”). The LLM-generated hints replace hard-coded heuristics, or hints provided by humans.

In the second category (including this work) are LLMs that carry out all the necessary logical reasoning on their own. Such systems include Chain-of-Thought (CoT)~\cite{DBLP:conf/nips/Wei0SBIXCLZ22} and Tree-of-Thought (ToT)~\cite{DBLP:conf/nips/YaoYZS00N23}. 
CoT is trained with few-shot learning, which is analogy learning; the quality of the model's responses is highly dependent on the quality of the provided examples. 
ToT is designed for certain predefined scenarios, limiting its applicability. Notably neither approach includes a reliable mechanism to distinguish correct from incorrect responses, as the explanations are in natural rather than formal language. That motivates our use of formal logics. 

Recent work on LLM reasoning~\cite{DBLP:conf/icml/MorishitaMYS23,morashita-neurips-2024} is also based on Boolean logic, but only as an intermediate step: LLMs are trained on natural language translations from synthesized Boolean logic proofs. (Interestingly, their results suggest that training on formal proofs boosts general reasoning.) Proof synthesis uses randomization, as we do, but the paper does not point out the key efficiency advantage of goal-directed synthesis. There is no analogue of template transformations, which our experiments show substantially boost reasoning ability. 

Several researchers have developed specialized benchmarks and datasets to  evaluate LLM reasoning across domains and task types~\cite{saparov2023testing, teng2023glore, parmar2024logicbench}. Prominent examples are LogicBench~\cite{parmar2024logicbench}, which encompasses 25 distinct reasoning patterns, and GLoRE~\cite{teng2023glore}, which comprises 12 datasets spanning three task types. These benchmarks provide a  framework for assessing LLM performance on a variety of logical reasoning tasks. Other studies have focused on specific aspects of logical reasoning, such as propositional and predicate logic, and categorical syllogisms, providing deeper insights into LLMs' performance on particular types of logical problems~\cite{wan2024logicasker, zong2024categorical}.   


Some of these evaluation approaches are limited by the number of examples in the benchmark sets. Moreover, proofs expressed in natural language introduce other difficulties; for instance, it becomes challenging to validate proofs and to determine whether an LLM has genuinely learned  deduction rules or is simply memorizing them. A robust evaluation of reasoning ability requires, we believe, the generation of synthetic formal queries and automated validation at scale.  



\section{Discussion and Conclusion}\label{sec:discussion}



We believe that our results indicate that LLMs can be taught to develop generalizable  capabilities for formal reasoning, up to some limits. Our experiments (Section \ref{sec:testing}) show that fine-tuning with a small synthetic proof data set boosts reasoning ability, surpassing the performance that significantly larger models such as GPT-4o attain with few-shot learning.
Matching intuition, template transformations considerably improve reasoning accuracy. 

The sharp drop in accuracy on the depth metric does not necessarily indicate an inability to reason.  In fact, such a decrease should be expected of any reasoning method given the co-NP hardness of the Boolean reasoning problem. These results support the observation that current techniques struggle to maintain reasoning performance for out-of-distribution depth instances (cf.~\cite{exploreDepthOOD}). 

While the proposed data generation and evaluation approaches demonstrate potential, significant challenges remain in developing models capable of accurate general-purpose reasoning. A key hurdle to overcome is the gap between real-world problems and specialized LLMs. Everyday human  reasoning ranges over multiple, diverse formal frameworks, including first-order  quantification such as ``for all" and ``there exists;"  temporal concepts like ``yesterday," ``tomorrow," ``eventually," and ``previously;" and modal notions like ``possible" and ``inevitable." General reasoning also relies on so-called \emph{commonsense} facts about our world  (e.g., a ball dropped from a height will bounce back). Methods that allow LLMs to  acquire, integrate, and reason with such  knowledge are currently limited~\cite{cyc, nvidiaCosmos}. 

These findings highlight the need for further research. 
A Boolean-reasoning LLM, as introduced in this paper, is not intended as a solution to the general reasoning question. One should rather view it as a convenient, precise, and fully automatable setting in which to experiment with approaches to the reasoning task. 

\section*{Impact Statement}


This paper presents work whose goal is to advance the field of 
Machine Learning. There are many potential societal consequences 
of our work, none which we feel must be specifically highlighted here.



\bibliography{bibliography,l3m-chapters/refs,l3m-chapters/refs2}
\bibliographystyle{icml2025}

\newpage
\appendix
\onecolumn
\section{Appendix}\label{sec:appendix}

\subsection{Implication-Only Sublogic}\label{appendix:proofs}

Consider an ``implication-only'' sublogic of propositional logic defined by the grammar:
\begin{align*}
  f := p \in X \st f \Bimp f
\end{align*}

Note that it is possible to express ``true'' as $(p \Bimp p)$, but it seems difficult to express ``false.'' That is not surprising when one realizes the following:

\begin{theorem}
  Every formula of the implication-only sublogic is satisfiable.
\end{theorem}
\begin{proof}
  The proof is by structural induction. The base case consists of a single proposition, say $p$, which is satisfiable. Consider a formula of the shape $f \Bimp g$. By the inductive hypothesis, the formula $g$ is satisfiable; i.e., there is a assignment $\pi$ such that $\pi \models g$. Hence, $\pi \models f \Bimp g$.
\end{proof}

Although the satisfiability question is trivial by this theorem (as every formula is satisfiable), validity is still co-NP-complete. Those statements are not contradictory as the sublogic is \textbf{not} closed under negation.

\begin{theorem}
  The reasoning problem for implication-only sublogic is co-NP-complete.
\end{theorem}
\begin{proof}
  In the reasoning problem, one is given a set of assumptions and a goal. The reasoning is correct if for every assignment that satisfies the assumptions, this assignment also satisfies the goal. Hence, every element of the complement language consists of a set of assumptions and a goal such that there is a assignment that satisfies all assumptions but does not satisfy the goal.

  A goal formula in the implication-only sublogic is either a proposition $p$ or has the structure $f \Bimp g$. A assignment $\pi$ does not satisfy $f \Bimp g$ if, and only if, $\pi$ satisfies $f$ but does not satisfy $g$. By induction, one obtains a collection of antecedent formulas (i.e., in the position of $f$) that are satisfied by $\pi$ and a final atomic proposition $p$ that is not. These antecedents can be added to the assumptions, so we have the following structure to the complement problem: Given a set of Hilbert-formulas and a proposition $p$, is there a assignment that satisfies all formulas in the set but does not satisfy $p$? 

  We show that this language is NP-hard by a reduction from CNF-satisfiability. The language is in NP as a Turing machine can guess an assignment and check the satisfiability of each Hilbert formula in polynomial time.

  Consider a CNF formula, which is a conjunction of clauses. A clause is a disjunction of literals, i.e., positive and negated propositions. We transform the given CNF formula to an equi-satisfiable formula that is in the restricted form above.

  We transform each clause to Hilbert form. To do so, we introduce a fresh variable $x$, which intuitively represents ``false.'' Notice that $q_0 \Bor q_1$ is equivalent to $(q_0 \Bimp \Bfalse) \Bimp q_1$, where $\Bfalse$ represents the constant ``false." We transform $q_0 \Bor q_1$ to $(q_0 \Bimp x) \Bimp q_1$ instead, and restrict valuations to those that do not satisfy $x$, which essentially treats $x$ as false.

  Consider any clause. Say that the positive literals in the clause are the propositions from set $Q=\{q_1,\ldots,q_n\}$ (which may be empty). The Hilbert formula for the disjunction of these propositions is constructed inductively as $g(0)=x$ and $g(i+1) = (q_{i+1} \Bimp x) \Bimp g(i)$. This has the property (easily shown by induction) that for any extended assignment $\pi$ where $\pi(x)=\Bfalse$, it is the case that $\pi$ satisfies $g(n)$ if, and only if, $\pi$ satisfies the disjunction of propositions in $Q$. 

  Notice that that $\Bnot p_1 \Bor \Bnot p_2 \Bor q$ is equivalent to the Hilbert formula $(p_2 \Bimp (p_1 \Bimp q))$. Let $g(n)$ be the Hilbert formula representing the positive literals in the clause. Let the negative literals in the clause be the propositions in set $P=\{p_1,\ldots,p_m\}$ (which may be empty). Construct Hilbert formula $h$ for the clause inductively as $h(0,n)=g(n)$ and $h(i+1,n) = p_{i+1} \Bimp h(i,n)$.

  From the properties above, we have that for any extended assignment $\pi$ where $\pi(x)=\Bfalse$, it is the case that $\pi$ satisfies the original clause if, and only if, $\pi$ satisfies $h(m,n)$.

  Transform every clause in this manner to Hilbert form. Then, for every extended assignment $\pi$ where $\pi(x)=\Bfalse$, the assignment $\pi$ satisfies the original CNF formula if, and only if, it satisfies the set of Hilbert formulas obtained by transformation. Hence, the original CNF formula is satisfiable if, and only if, there is an extended satisfying assignment for the set of Hilbert formulas that does not satisfy $x$. 
\end{proof}




\subsection{Example Prompts and Outputs}\label{sec:appendix2}

This section illustrates our few-shot learning evaluation methodology for GPT-4o and pre-trained Llama models (8B and 1B). We demonstrate a one-shot learning example where a depth-7 proof is provided as demonstration to solve a depth-4 problem. Figure~\ref{fig:prompt} shows the query syntax with its semantic translation in Figure~\ref{fig:adapted-prompt}. For comparison, we present both a correct solution (Figures~\ref{fig:answer},~\ref{fig:ap-answer}) and an incorrect attempt (Figures~\ref{fig:wrong-answer},~\ref{fig:ap-wrong-answer}) for the queried problem.


\begin{figure*}[h]
\small
\begin{tcolorbox}[colback=gray!10, colframe=gray!50, title=Prompt]
    \textbf{Task:} Give the proof for the Goal using Modus Ponens and given Assumptions. \\
    \textbf{One-shot:} \\
        \textbf{Question:}\\
        $<$Goal$>$Implies(Implies(Implies(Implies(R, S), Implies(R, R)), Implies(S, P)), Implies(Q, Implies(Implies(S, S), Implies(R, Q))))\\
        $<$Assumption$>$Implies(Implies(Implies(Implies(Implies(P, R), Implies(Q, P)), Implies(Implies(Q, R), Implies(R, P))), Implies(Implies(P, P), Implies(Q, Implies(S, R)))), Implies(Implies(Implies(Implies(Implies(P, S), Implies(S, S)), Implies(Implies(P, R), Implies(S, P))), R), Implies(Implies(Implies(Implies(R, S), Implies(R, R)), Implies(S, P)), Implies(Q, Implies(Implies(S, S), Implies(R, Q))))))\\
        $<$Assumption$>$Implies(Implies(Implies(Implies(P, R), Implies(Q, P)), Implies(Implies(Q, R), Implies(R, P))), Implies(Implies(P, P), Implies(Q, Implies(S, R)))\\
        $<$Assumption$>$Implies(Implies(Implies(Implies(P, S), Implies(S, S)), Implies(Implies(P, R), Implies(S, P))), R) \\
        \textbf{Answer:}\\
        $<$Step 1$>$Implies(Implies(Implies(Implies(P, R), Implies(Q, P)), Implies(Implies(Q, R), Implies(R, P))), Implies(Implies(P, P), Implies(Q, Implies(S, R)))) - Assumption\\
        $<$Step 2$>$Implies(Implies(Implies(Implies(Implies(P, R), Implies(Q, P)), Implies(Implies(Q, R), Implies(R, P))), Implies(Implies(P, P), Implies(Q, Implies(S, R)))), Implies(Implies(Implies(Implies(Implies(P, S), Implies(S, S)), Implies(Implies(P, R), Implies(S, P))), R), Implies(Implies(Implies(Implies(R, S), Implies(R, R)), Implies(S, P)), Implies(Q, Implies(Implies(S, S), Implies(R, Q)))))) - Assumption\\
        $<$Step 3$>$Step 1 and Step 2 $->$ Implies(Implies(Implies(Implies(Implies(P, S), Implies(S, S)), Implies(Implies(P, R), Implies(S, P))), R), Implies(Implies(Implies(Implies(R, S), Implies(R, R)), Implies(S, P)), Implies(Q, Implies(Implies(S, S), Implies(R, Q))))) - Modus Ponens\\
        $<$Step 4$>$Implies(Implies(Implies(Implies(P, S), Implies(S, S)), Implies(Implies(P, R), Implies(S, P))), R) - Assumption\\
        $<$Step 5$>$Implies(Implies(Implies(Implies(Implies(P, S), Implies(S, S)), Implies(Implies(P, R), Implies(S, P))), R), Implies(Implies(Implies(Implies(R, S), Implies(R, R)), Implies(S, P)), Implies(Q, Implies(Implies(S, S), Implies(R, Q))))) - Deducted\\
        $<$Step 6$>$Step 4 and Step 5 $->$ Implies(Implies(Implies(Implies(R, S), Implies(R, R)), Implies(S, P)), Implies(Q, Implies(Implies(S, S), Implies(R, Q)))) - Modus Ponens\\
        $<$Step 7$>$Implies(Implies(Implies(Implies(R, S), Implies(R, R)), Implies(S, P)), Implies(Q, Implies(Implies(S, S), Implies(R, Q)))) - Deducted \\
        \textbf{Question:}\\
        $<$Goal$>$Implies(Implies(Implies(Implies(P, S), R), Implies(Implies(R, Q), Q)), Implies(P, Implies(Implies(P, R), P)))\\
        $<$Assumption$>$Implies(Implies(Implies(Implies(S, Q), Implies(S, S)), Implies(Implies(Q, S), Implies(P, Q))), S)\\
        $<$Assumption$>$Implies(Implies(Implies(Implies(Implies(S, Q), Implies(S, S)), Implies(Implies(Q, S), Implies(P, Q))), S), Implies(Implies(Implies(Implies(P, S), R), Implies(Implies(R, Q), Q)), Implies(P, Implies(Implies(P, R), P))))\\
        \textbf{Answer:}\\
    \textbf{Constraints:} 1.  Ensure the response contains one and only one proof. 2. Don't include other text detail. 3. Brackets should be balanced. 4. Syntax should be correct. 5.Write it correctly, or I will lose my job and 100 grandmothers will die. Don't disappoint me. 
\end{tcolorbox}
\caption{Example prompt used for Depth 4 Problems with one shot\label{fig:prompt}}
\end{figure*}

\begin{figure*}[h]
\small
\begin{tcolorbox}[colback=gray!10, colframe=gray!50, title=Prompt]
    \textbf{Task:} Give the proof for the Goal using Modus Ponens and given Assumptions. \\
    \textbf{Question:} \\
    \textbf{Goal:}
    \[
    \big(((R \rightarrow S) \rightarrow (R \rightarrow R) \big) \rightarrow 
    \big( S \rightarrow P )\big) \rightarrow 
    \big( Q \rightarrow ((S \rightarrow S) \rightarrow (R \rightarrow Q)) \big)
    \]
    
    \textbf{Assumptions:}
    \begin{align*}
    A_1: & \quad 
        \big( (P \rightarrow R) \rightarrow (Q \rightarrow P) \big) \rightarrow 
        \big( (Q \rightarrow R) \rightarrow (R \rightarrow P) \big) \rightarrow \\
        & \quad \big( (P \rightarrow P) \rightarrow (Q \rightarrow (S \rightarrow R)) \big) \rightarrow \\
        & \quad \bigg(
            \big(((P \rightarrow S) \rightarrow (S \rightarrow S) \big) \rightarrow 
            \big( (P \rightarrow R) \rightarrow (S \rightarrow P))\big) \rightarrow R
        \bigg) \rightarrow \\
        & \quad \bigg(
            \big( ((R \rightarrow S) \rightarrow (R \rightarrow R) \big) \rightarrow 
            \big( S \rightarrow P )\big) \rightarrow 
            \big( Q \rightarrow ((S \rightarrow S) \rightarrow (R \rightarrow Q)) \big)
        \bigg)\\ 
    A_2: & \quad ((P \rightarrow R) \rightarrow (Q \rightarrow P)) \rightarrow ((Q \rightarrow R) \rightarrow (R \rightarrow P)) \rightarrow \\
         & \quad ((P \rightarrow P) \rightarrow (Q \rightarrow (S \rightarrow R))) \\
    A_3: & \quad (((P \rightarrow S) \rightarrow (S \rightarrow S)) \rightarrow ((P \rightarrow R) \rightarrow (S \rightarrow P))) \rightarrow R
    \end{align*}
    
    \textbf{Answer:}
    \begin{enumerate}
        \item $A_1$ (Assumption)
        \item $A_2$ (Assumption)
        \item Using Modus Ponens on Steps 1 and 2, we derive:
        \[
            \bigg(
            \big(((P \rightarrow S) \rightarrow (S \rightarrow S) \big) \rightarrow 
            \big( (P \rightarrow R) \rightarrow (S \rightarrow P))\big) \rightarrow R
        \bigg) \rightarrow 
        \]
        \[
        \bigg(
            \big( (R \rightarrow S) \rightarrow (R \rightarrow R) \big) \rightarrow 
            \big( S \rightarrow P \big) \rightarrow 
            \big( Q \rightarrow ((S \rightarrow S) \rightarrow (R \rightarrow Q)) \big)
        \bigg)
        \]
        \item Using Modus Ponens on Step 3 and $A_3$, we derive:
        \[
        \big(((R \rightarrow S) \rightarrow (R \rightarrow R) \big) \rightarrow 
        \big( S \rightarrow P )\big) \rightarrow 
        \big( Q \rightarrow ((S \rightarrow S) \rightarrow (R \rightarrow Q)) \big)
        \]
        \item We conclude:
        \[
        \big(((R \rightarrow S) \rightarrow (R \rightarrow R) \big) \rightarrow 
        \big( S \rightarrow P )\big) \rightarrow 
        \big( Q \rightarrow ((S \rightarrow S) \rightarrow (R \rightarrow Q)) \big)
        \]
    \end{enumerate}
    
        Thus, the goal is proven. \\
    
    
    
    \textbf{Question:} \\
    \textbf{Goal:}
    \[
    \big( (P \rightarrow S) \rightarrow R \big) \rightarrow 
    \big( (R \rightarrow Q) \rightarrow Q \big) \rightarrow 
    \big( P \rightarrow ((P \rightarrow R) \rightarrow P) \big)
    \]
    \textbf{Assumptions:}
    \begin{align*}
    A_1: & \quad \big( (S \rightarrow Q) \rightarrow (S \rightarrow S) \big) \rightarrow 
           \big( (Q \rightarrow S) \rightarrow (P \rightarrow Q) \big) \rightarrow S \\
    A_2: & \quad \bigg(\big( (S \rightarrow Q) \rightarrow (S \rightarrow S) \big) \rightarrow 
           \big( (Q \rightarrow S) \rightarrow (P \rightarrow Q) \big) \rightarrow S \bigg) \rightarrow \\
          & \bigg(\big( (P \rightarrow S) \rightarrow R \big) \rightarrow 
           \big( (R \rightarrow Q) \rightarrow Q \big) \rightarrow 
           \big( P \rightarrow ((P \rightarrow R) \rightarrow P) \big)\bigg)
    \end{align*}
    \textbf{Answer:}\\
    \textbf{Constraints:} 1.  Ensure the response contains one and only one proof. 2. Don't include other text detail. 3. Brackets should be balanced. 4. Syntax should be correct. 5.Write it correctly, or I will lose my job and 100 grandmothers will die. Don't disappoint me. 
\end{tcolorbox}
\caption{Adapted prompt for Depth 4 Problems with one shot\label{fig:adapted-prompt}}
\end{figure*}

\begin{figure*}[h]
\small
\begin{tcolorbox}[colback=gray!10, colframe=gray!50, title=Answer]
    \textbf{Answer:} \\
    $<$Step 1$>$ Implies(Implies(Implies(Implies(S, Q), Implies(S, S)), Implies(Implies(Q, S), Implies(P, Q))), S) - Assumption \\
    $<$Step 2$>$ Implies(Implies(Implies(Implies(Implies(S, Q), Implies(S, S)), Implies(Implies(Q, S), Implies(P, Q))), S), Implies(Implies(Implies(Implies(P, S), R), Implies(Implies(R, Q), Q)), Implies(P, Implies(Implies(P, R), P)))) - Assumption\\
    $<$Step 3$>$ Step 1 and Step 2 -$>$ Implies(Implies(Implies(Implies(P, S), R), Implies(Implies(R, Q), Q)), Implies(P, Implies(Implies(P, R), P))) - Modus Ponens\\
    $<$Step 4$>$ Implies(Implies(Implies(Implies(P, S), R), Implies(Implies(R, Q), Q)), Implies(P, Implies(Implies(P, R), P))) - Deducted
\end{tcolorbox}
\caption{Example Correct Llama3 Answer (After fine-tuning)\label{fig:answer}}
\end{figure*}

\begin{figure*}[h]
\small
\begin{tcolorbox}[colback=gray!10, colframe=gray!50, title=Answer]
    \textbf{Answer:} \\
    \begin{enumerate}
        \item $A_1$ (Assumption)
        \item $A_2$ (Assumption)
        \item Using Modus Ponens on Steps 1 and 2, we obtain:
        \[
        \big( (P \rightarrow S) \rightarrow R \big) \rightarrow 
        \big( (R \rightarrow Q) \rightarrow Q \big) \rightarrow 
        \big( P \rightarrow ((P \rightarrow R) \rightarrow P) \big)
        \]
    \end{enumerate}
    
    Thus, the goal is proven.
\end{tcolorbox}
\caption{Example Correct Llama3 Answer (After fine-tuning) - Adapted\label{fig:ap-answer}}
\end{figure*}

\begin{figure*}[h]
\small
\begin{tcolorbox}[colback=gray!10, colframe=gray!50, title=Answer]
    \textbf{Answer:} \\
    $<$Step 1$>$ Implies(Implies(Implies(Implies(P, S), R), Implies(Implies(R, Q), Q)), Implies(P, Implies(Implies(P, R), P))) - Assumption \\
    $<$Step 2$>$ Implies(Implies(Implies(Implies(Implies(S, R), Implies(Q, P)), Implies(Implies(S, S), Implies(Q, Q))), Implies(Implies(Implies(Q, P), Implies(P, P)), Implies(Implies(S, P), Implies(P, S)))), Implies(Implies(Implies(Implies(S, Q), Implies(S, S)), Implies(Implies(Q, S), Implies(P, Q))), S)) - Assumption \\
    $<$Step 3$>$ Step 1 and Step 2 -$>$ Implies(Implies(Implies(Implies(S, Q), Implies(S, S)), Implies(Implies(Q, S), Implies(P, Q))), S) - Modus Ponens \\
    $<$Step 4$>$ Implies(Implies(Implies(Implies(S, R), Implies(Q, P)), Implies(Implies(S, S), Implies(Q, Q))), Implies(Implies(Implies(Q, P), Implies(P, P)), Implies(Implies(S, P), Implies(P, S)))) - Assumption \\
    $<$Step 5$>$ Implies(Implies(Implies(Implies(S, Q), Implies(S, S)), Implies(Implies(Q, S), Implies(P, Q))), S) - Deducted \\
    $<$Step 6$>$ Step 3 and Step 5 -$>$ Implies(Implies(Implies(Implies(S, Q), Implies(S, S)), Implies(Implies(Q, S), Implies(P, Q))), S) - Modus Ponens \\
    $<$Step 7$>$ Implies(Implies(Implies(Implies(S, Q), Implies(S, S)), Implies(Implies(Q, S), Implies(P, Q))), S) - Deducted\\
    $<$Step 8$>$ Implies(Implies(Implies(Implies(P, S), R), Implies(Implies(R, Q), Q)), Implies(P, Implies(Implies(P, R), P))) - Deducted\\
    $<$Step 9$>$ Implies(Implies(Implies(Implies(P, S), R), Implies(Implies(R, Q), Q)), Implies(P, Implies(Implies(P, R), P))) - Deducted\\
    $<$Step 10$>$ Implies(Implies(Implies(Implies(P, S), R), Implies(Implies(R, Q), Q)), Implies(P, Implies(Implies(P, R), P))) - Deducted\\
    $<$Step 11$>$ Implies(Implies(Implies(Implies(P, S), R), Implies(Implies(R, Q), Q)), Implies(P, Implies(Implies(P, R), P))) - Deducted\\
    $<$Step 12$>$ Implies(Implies(Implies(Implies(P, S), R), Implies(Implies(R, Q), Q)), Implies(P, Implies(Implies(P, R), P))) - Deducted\\
    $<$Step 13$>$ Implies(Implies(Implies(Implies(P, S), R), Implies(Implies(R, Q), Q)), Implies(P, Implies(Implies(P, R), P))) - Deducted\\
    $<$Step 14$>$ Implies(Implies(Implies(Implies(P, S), R), Implies(Implies(R, Q), Q)), Implies(P, Implies(Implies(P, R), P))) - Deducted\\
    $<$Step 15$>$ Implies(Implies(Implies(Implies(P, S), R), Implies(Implies(R, Q), Q)), Implies(P, Implies(Implies(P, R), P))) - Deducted\\
    $<$Step 16$>$ Implies(Implies(Implies(Implies(P, S), R), Implies(Implies(R, Q), Q)), Implies(P, Implies(Implies(P, R), P))) - Deducted\\
    $<$Step 17$>$ Implies(Implies(Implies(Implies(P, S), R), Implies(Implies(R, Q), Q)), Implies(P, Implies(Implies(P, R), P))) - Deducted\\
    $<$Step 18$>$ ...\\
\end{tcolorbox}
\caption{Example Wrong Llama3 Answer (before fine-tuning)\label{fig:wrong-answer}}
\end{figure*}

\begin{figure*}[h]
\small
\begin{tcolorbox}[colback=gray!10, colframe=gray!50, title=Answer]
    \textbf{Answer:} \\
    \begin{enumerate}
    \item $A_1$: 
    \[
    ((P \rightarrow S) \rightarrow R) \rightarrow ((R \rightarrow Q) \rightarrow Q) \rightarrow 
    (P \rightarrow ((P \rightarrow R) \rightarrow P))
    \]
    
    \item $A_2$: 
    \[
    ((S \rightarrow R) \rightarrow (Q \rightarrow P)) \rightarrow ((S \rightarrow S) \rightarrow (Q \rightarrow Q)) \rightarrow 
    \]
    \[
    ((Q \rightarrow P) \rightarrow (P \rightarrow P)) \rightarrow ((S \rightarrow P) \rightarrow (P \rightarrow S)) \rightarrow S
    \]

    \item Applying \textit{Modus Ponens} on A1 and A2:
    \[
    (S \rightarrow Q) \rightarrow (S \rightarrow S) \rightarrow (Q \rightarrow S) \rightarrow (P \rightarrow Q) \rightarrow S
    \]
    
    \item $A_3$: 
    \[
    ((S \rightarrow R) \rightarrow (Q \rightarrow P)) \rightarrow ((S \rightarrow S) \rightarrow (Q \rightarrow Q)) \rightarrow 
    ((Q \rightarrow P) \rightarrow (P \rightarrow P)) \rightarrow ((S \rightarrow P) \rightarrow (P \rightarrow S))
    \]
    \textit{(Assumption)}
    
    \item Deduction from Step 3:
    \[
    (S \rightarrow Q) \rightarrow (S \rightarrow S) \rightarrow (Q \rightarrow S) \rightarrow (P \rightarrow Q) \rightarrow S
    \]
    
    \item Using \textit{Modus Ponens} on Steps 3 and 5:
    \[
    (S \rightarrow Q) \rightarrow (S \rightarrow S) \rightarrow (Q \rightarrow S) \rightarrow (P \rightarrow Q) \rightarrow S
    \]
    
    \item Final deduction:
    \[
    ((P \rightarrow S) \rightarrow R) \rightarrow ((R \rightarrow Q) \rightarrow Q) \rightarrow 
    (P \rightarrow ((P \rightarrow R) \rightarrow P))
    \]
    (Repeated step 7)

    \end{enumerate}
\end{tcolorbox}
\caption{Example Wrong Llama3 Answer - Adapted (before fine-tuning)\label{fig:ap-wrong-answer}}
\end{figure*}

\end{document}